\documentclass{article}

\usepackage[preprint]{neurips_2023}

\usepackage[utf8]{inputenc} %
\usepackage[T1]{fontenc}    %
\usepackage{url}            %
\usepackage{booktabs}       %
\usepackage{amsfonts}       %
\usepackage{nicefrac}       %
\usepackage{microtype}      %
\usepackage{xcolor}         %
\setcitestyle{square,comma,numbers}

\title{Kissing to Find a Match:\\
Efficient Low-Rank Permutation Representation}

\author{%
  Hannah Dröge
    \\
  University of Siegen\\
  57076 Siegen, Germany \\
  \texttt{hannah.droege@uni-siegen.de} \\
  \And
  Zorah Lähner \\
  University of Siegen \\
  57076 Siegen, Germany \\
  \texttt{zorah.laehner@uni-siegen.de} \\
  \And
  Yuval Bahat \\
  Princeton University  \\
  Princeton, NJ 08544, United States \\
  \texttt{yuval.bahat@gmail.com} \\
  \And
  Onofre Martorell \\
  University of Balearic Islands \\
  Investigador ForInDoc del Govern de les Illes Balears \\
  07122 Palma, Illes Balears, Spain \\
  \texttt{o.martorell@uib.cat} \\
  \And
  Felix Heide \\
  Princeton University \\
  Princeton, NJ 08544, United States \\
  \texttt{fheide@cs.princeton.edu} \\
  \And
  Michael Möller \\
  University of Siegen\\
  57076 Siegen, Germany \\
  \texttt{michael.moeller@uni-siegen.de} \\
}

\usepackage{caption}
\usepackage{subcaption}
\usepackage{mathtools,amsthm}
\newtheorem{proposition}{Proposition}

\newtheorem{definition}{Definition}
\newcommand{\kiss}[1]{\text{Kiss}(#1)}
\newcommand{\Perm}{\mathcal{P}}
\usepackage{booktabs}
\usepackage{multirow}

\newcommand{\inA}{{V}}
\newcommand{\inB}{{W}}
\newcommand{\mata}{{V}} %
\newcommand{\matb}{{W}}
\newcommand{\perm}{P}
\DeclareMathOperator*{\argmax}{arg\,max}
\DeclareMathOperator*{\argmin}{arg\,min}
\newcommand{\errorTwo}{{e_{emb}}}
\newcommand{\errorOne}{{e_{prob}}}

\usepackage{pgfplots}
\usepackage{pgf,tikz,tikzscale}
\newcommand{\etal}{\textit{et al.~}}
\usetikzlibrary{through,calc,intersections}

\usepackage{tikz}
\usetikzlibrary{fit}

\tikzset{%
  highlight/.style={rectangle,rounded corners,fill=red!15,draw,fill opacity=0.5,thick,inner sep=0pt}
}

\usepackage{nicematrix}

\usepackage{hyperref}       %
\begin{document}

\maketitle

\begin{abstract}
Permutation matrices play a key role in matching and assignment problems across the fields, especially in computer vision and robotics. However, memory for explicitly representing permutation matrices grows quadratically with the size of the problem, prohibiting large problem instances. In this work, we propose to tackle the curse of dimensionality of large  permutation matrices by approximating them using low-rank matrix factorization, followed by a nonlinearity. To this end, we rely on the Kissing number theory to infer the minimal rank required for representing a permutation matrix of a given size, which is significantly smaller than the problem size. This leads to a drastic reduction in computation and memory costs, e.g., up to $3$ orders of magnitude less memory for a problem of size $n=20000$, represented using $8.4\times10^5$ elements in two small matrices instead of using a single huge matrix with $4\times 10^8$ elements. The proposed representation allows for accurate representations of large permutation matrices, which in turn enables handling large problems that would have been infeasible otherwise. We demonstrate the applicability and merits of the proposed approach through a series of experiments on a range of problems that involve predicting permutation matrices, from linear and quadratic assignment to shape matching problems.

\end{abstract}

\section{Introduction}
Permutation matrices, which encode the reordering of elements, arise naturally in any problem that can be phrased as a bijection between two equally sized sets. As such, they are
fundamental to many important computer vision applications, including matching semantically identical key points in images \cite{zanfir2018deep, wang2020combinatorial, wang2021neural,yu2020learning}, matching 3D shapes or point clouds \cite{kezurer15tight,vestner2017kernel,maron2016point}, estimating scene flow on point clouds \cite{puy2020flot} and solving jigsaw puzzles \cite{mena2018learning}, as well as to various sorting tasks~\cite{bernard2018ds,grover2019stochastic}.  

A permutation $p$, corresponding to the bijection from the set $\{1,\hdots,n\}$ onto itself, can be represented efficiently without a permutation matrix by merely enumerating the $n$ elements
\begin{align}
\label{eq:enumerationRepresentation}
    (p(1),~p(2), \hdots, p(n)) \in \mathbb{N}^n.
\end{align}
However, this representation is unsuitable for most computer vision problems that involve estimating $p$ through optimization since this representation (i) is inherently discrete, yielding combinatorial problems for which no natural relaxation exists, and (ii) induces a solution space with a meaningless distance metric, as element $i$ in the set generally is not 'closer' to element $i+1$ than it is to any other element $j$.

As a result, almost all methods for predicting permutations, including learning-based methods, favor a \textit{permutation matrix} representation instead, i.e., formulating a permutation as an element in the set 
\begin{align}
\label{eq:setOfPermutationMatrices}
    \Perm_n = \{ P \in \{0,1\}^{n \times n} ~|~ &\sum_{i}P_{ij} = 1,~\sum_{j}P_{ij} = 1 ~\forall i,j \},
\end{align}
with $p(i)=j$ in representation \eqref{eq:enumerationRepresentation} corresponding to $P_{ij}=1$ in the matrix representation form \eqref{eq:setOfPermutationMatrices}, which allows predicting $p$ via optimization methods.
\begin{figure}
    \centering
    \definecolor{blue}{rgb}{0.00000, 0.00000, 0.60000}%
\definecolor{red}{rgb}{0.00000, 0.60000, 0.00000}%
\definecolor{cc}{rgb}{0.60000, 0.00000, 0.00000}%
\definecolor{regalia}{rgb}{0.00000, 0.60000, 0.60000}%
\definecolor{cc}{rgb}{0.60000, 0.00000, 0.60000}%

\scalebox{0.85}{
\begin{tikzpicture}[scale=0.9]

\draw[black]  (0,0) circle (1);

\node (circle) at (0.55,1.7) {\color{red}$V_j$};
\node (circle) at (-0.55,1.7) {\color{blue}$V_i$};

\draw[->,black](1.0,0.0) -- (2.0,0.0);
\draw[->,red](0.4999999701976776,0.8660253882408142) -- (0.9999999403953552,1.7320507764816284);
\draw[dashed,red](0,0) -- (0.4999999701976776,0.8660253882408142) ;
\draw[->,blue](-0.5,0.8660253286361694) -- (-1.0,1.7320506572723389);
\draw[dashed,blue](0,0) -- (-0.5,0.8660253286361694) ;
\draw[->,black](-0.9999999403953552,-5.960464477539063e-08) -- (-1.9999998807907104,-1.1920928955078125e-07);
\draw[->,black](-0.49999988079071045,-0.8660253286361694) -- (-0.9999997615814209,-1.7320506572723389);
\draw[->,black](0.5,-0.8660252094268799) -- (1.0,-1.7320504188537598);

\draw[fill=gray!30] (0,0) -- ( 60 : 0.4 ) arc ( 60:120:0.4 )  node [above,pos=0.5] {$\alpha$};

  \node at (9.5,0.0) {
$ {2}\max{\left(
\begin{pNiceMatrix}
1.0 & 0.0\\
0.5 & 0.87\\
-0.5 & 0.87\\
-1.0 & -0.0\\
-0.5 & -0.87\\
0.5 & -0.87\\
\end{pNiceMatrix}
\cdot
\begin{pNiceMatrix}
-1.0 & -0.0\\
1.0 & 0.0\\
{\color{red} 0.5} & {\color{red} 0.87}\\
-0.5 & -0.87\\
0.5 & -0.87\\
{\color{blue} -0.5} &{\color{blue} 0.87}\\
\end{pNiceMatrix}^T - \frac{1}{2},0\right)}
= \begin{pNiceMatrix}
0& 1& 0& 0& 0& 0&\\
0& 0& {\color{red} 1}& 0& 0& {\color{regalia} 0}&\\
0& 0& {\color{regalia} 0}& 0& 0& {\color{blue} 1}&\\
1& 0& 0& 0& 0& 0&\\
0& 0& 0& 1& 0& 0&\\
0& 0& 0& 0& 1& 0&\\
\end{pNiceMatrix}$
  };
\end{tikzpicture}
}
    \caption{Geometric intuition behind our approach on a 2D unit sphere. For well-distributed vectors $V \in \mathbb{R}^{\kiss{2} \times 2}$,
    where the number of vectors is determined by the Kissing number ($\kiss{2} = 6$),
    the cosine angle between different vectors $V_{i,:}$ and $V_{j,:}$, $i \neq j$, is $\langle V_{i,:}, V_{j,:}\rangle = \cos(\alpha) \leq 0.5$, while $\langle V_{i,:}, V_{i,:}\rangle = 1$ for the same vector. 
Thus, for any permutation $P$, the matrix-matrix product of $V$ and $(PV)^T$ merely has to be thresholded suitably to represent the permutation $P$, i.e. $P = 2\max(V(PV)^T - 0.5,0)$.} 
    \label{fig:teaser}
\end{figure}
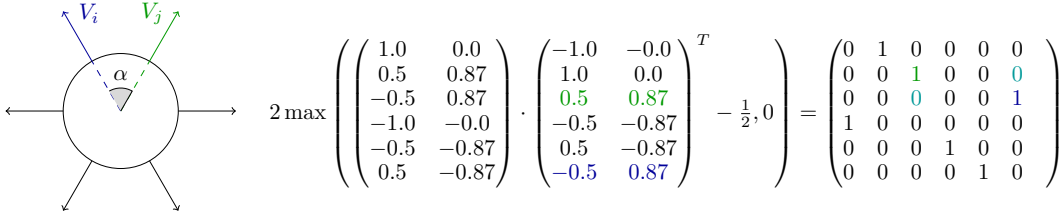
Yet, the advantages of the matrix form representation \eqref{eq:setOfPermutationMatrices} come at the cost of a prohibitive increase in memory, as it requires storing $n^2$ binary numbers $P_{ij}\in \{0,1\}$, or -- after commonly used relaxations -- even $n^2$-many floating point numbers instead of the $n$ integers in \eqref{eq:enumerationRepresentation}. This renders matching problems with $n>~\sim10^4$ largely infeasible as their corresponding permutation matrix $P$ constitutes over one hundred million entries.
To handle large matrices whose size prohibits explicit processing and storage, existing approaches typically either turn to \textit{sparse representations}, i.e., storing only a small portion of matrix values in the form of $(i,j,P_{i,j})$ triplets, where $P_{i,j}\neq0$, or employ \textit{low rank representations}, i.e., forming a large matrix $P$ as a product of matrices
\begin{align}
    \label{eq:lowRank}
    P = \mata \matb^T,
\end{align}
with $\mata,\matb \in \mathbb{R}^{n\times m}$ and $m<<n$.
Unfortunately, neither of these approaches is applicable to permutation matrices: sparse representation cannot be used as the sparsity pattern is not only unknown a-priori but actually the sought-after solution to the problem. On top of that, since permutation matrices are by definition full rank, a low-rank representation \eqref{eq:lowRank} can yield only a crude approximation at best. %

\paragraph*{Contributions. } In this work, we alleviate the limitation on problem size by harnessing the well-studied problem of (bounds for) the so-called \textit{Kissing number}, which, in practice, translates to introducing a simple adaptation to the matrix factorization approach \eqref{eq:lowRank}. In particular, we exploit the fact that for row-normalized matrices $\mata$ and $\matb$, the entries of $\mata \matb^T$ correspond to the cosines of the angles between the matrix rows. We then apply a pointwise non-linearity on the product of the matrices in \eqref{eq:lowRank}, which allows representing any permutation while using $m<<n$. We use the Kissing number theory to provide an estimate for how small an $m$ we can use. We elaborate on these theoretical considerations in Section \ref{sec:method} and provide an illustration of the geometric intuition for our approach in Fig.~\ref{fig:teaser}.

We then demonstrate the applicability of the proposed approach through several numerical experiments tackling various problems that involve estimating permutations, including a study on point alignment, linear assignment problems (LAPs), quadratic assignment problems (QAPs), and a real-world shape matching application. 
We find that the proposed approach trades off only little accuracy to offer significant memory saving, thus enabling handling bijection mapping problems that are %
larger than was previously possible, when full permutation matrices had to be stored.

\section{Related Work}\label{sec:soa}
\subsection{Permutation Learning and Representation}
Permutation learning aims to develop a model capable of predicting the optimal permutation matrix that matches a given input data or label.
Previous studies have suggested relaxing permutation matrices to continuous domains, such as the Birkhoff polytope \cite{birkhoff1946tres}, to approximate solutions for these optimization problems but still faced the problem of enforcing the sum-to-one constraints, which are commonly approximated using Sinkhorn layers \cite{prescott2011ranking}. 
In 2011, Adams \etal \cite{prescott2011ranking} proposed the Sinkhorn propagation method, with the goal of iteratively learning doubly-stochastic matrices that allow the prediction of permutation matrices. %
Following the same research line, Cruz \etal \cite{santa2017deeppermnet} proposed Sinkhorn networks, followed by
 Mena \etal \cite{mena2018learning} proposing Gumble-Sinkhorn networks, adding Gumbel noise to the Sinkhorn operator to learn latent permutations.
Grover \etal \cite{grover2019stochastic} made further efforts in the relaxation of a learned permutation matrix and propose to relax the matrix to become unimodal row-stochastic, meaning that the rows must maintain a total sum of one while having a distinct argmax. More recent studies suggest to circumvent the constraint of row and column sums being one by predicting permutations in Lehmer code, whose matrix form results to be row-stochastic \cite{diallo2020permutation}.
Other works propose alternative representations of permutations in different domains, as to work in the Fourier domain to give a compact representations of distribution over permutations \cite{huang2009fourier,pmlr-v2-kondor07a}, or embed them to the surface of a hypersphere \cite{plis2011directional}.

\subsection{Assignment Problems}

The goal of assignment problems is to find a permutation between two ordered sets while minimizing the assignment costs. 
The two most common versions are the \emph{linear} and \emph{quadratic} assignment problem (LAP and QAP) 
which are based on element-wise and pair-wise costs, respectively.
An LAP can be solved in cubic time using the Hungarian algorithm \cite{kuhn1955hungarian}.
The QAP was first introduced by Koopmans and Beckmann \cite{koopmansbeckmann} and can be written as
\begin{align}
    \min_{P \in \mathcal{P}_n} \text{tr}(A P B P^\top) + \text{tr}(C^\top P),
\end{align}
for  $A,B \in \mathbb{R}^{n \times n}$ and $C \in \mathbb{R}^n$ being problem
where $A \in \mathbb{R}^{n \times n}$ is the cost function between elements of the first object to match, $B \in \mathbb{R}^{n \times n}$ is the distance function between elements in the second object and $C \in \mathbb{R}^n$ is a point-wise energy.  The QAP has been proved to be NP-hard so no polynomial time solution can be expected for general cases.
As a result, many relaxations of the problem exist, for example by relaxing the permutation constraint \cite{gold1996graduated,rodola2012gametheoretic}, or by lifting the problem to higher dimensions \cite{kezurer15tight, zhao1998semidefinite}. 
A survey on various relaxation approaches can be found in \cite{loiola2007survey}. 
While the relaxations do ease some aspects of the problems, they normally do not decrease the dimensionality of the problem which remains demanding for large $n$. 

\subsection{3D Shape Correspondence}

3D shape correspondence is also often posed as an assignment problem between the sets of vertices of a pair of shapes, for example through point descriptors matched by an LAP or in an QAP aiming to preserve distances between all point pairs.
However, 3D shapes are often discretized with thousands of vertices which makes optimization for a permutation computationally challenging. 
Hence, the permutation constraint is often relaxed \cite{rodola2012gametheoretic} and, even though the tightness of relaxation might be known \cite{bernard2018ds,Dym:2017ue}, the optimization variables still scale quadratically with the problem size.
In \cite{gao2021multi} and \cite{vestner2017kernel}  the QAP is deconstructed into smaller problems and then each of them is optimized with a series of LAPs, while \cite{seelbach2021qmatch} solve for permutations as a series of cycles that gradually improve the solution.

Because permutation constraints for large resolution become infeasible, and, hence, the restriction to cases with the same number of vertices, recent methods often do not impose these constraints at all.
The functional maps framework \cite{ovsjanikov2012functional} converts the output to a mapping between functions on the shapes instead of vertices and can reduce the dimensionality drastically by using the isometry invariance property of the eigenfunctions of the Laplace-Beltrami operator \cite{pinkall1993}. 
Other lines of work rely on a given template to constrain the solution \cite{groueix2018b,sundararaman2022implicit}, impose physical models that regularizes the deformation between inputs \cite{eisenberger2021neuromorph, eisenberger2019divfree, papazov2011deformable}, or learn a solution space through training \cite{cao2023selfsupervised, litany2017deepfm,marin2020correspondence}.
However, with the exception of template-based models, these cannot guarantee permutations.

\section{Low-Rank Permutation Matrix Representation} \label{sec:method}
A common approach to solve optimization problems with costs $E$ over the set of permutation matrices $\mathcal{\perm}_n$ (including those arising from training neural networks for predicting assignments) is to relax the problem by replacing $\mathcal{\perm}_n$ by its convex hull $\text{conv}(\mathcal{\perm}_n)$, i.e., the set of doubly-stochastic matrices: 
\begin{equation}
    \label{eq:energyMin}
    \min_{P\in \text{conv}(\mathcal{\perm}_n)} E(P). 
\end{equation}
Since $P$ grows quadratically in $n$, has an unknown sparsity pattern, and the true solution is always full rank, such problems pose significant challenges for large $n$. In this work, we make the interesting observation that a non-linearity as simple as a rectified linear unit (ReLU, denoted by $\sigma$) is sufficient not only to restore a full rank, but to represent any permutation matrix exactly. More precisely, we propose to 
replace the set $\text{conv}(\mathcal{\perm}_n)$ in \eqref{eq:energyMin} with the set 
\mbox{$
 \mathcal{K}_m(\mathcal{\perm}_n) =\{\sigma(2VW^T-1) ~|~ V,W \in \mathbb{R}^{n\times m}\} 
$}
and use the so-called \textit{Kissing number} \cite{boyvalenkov2015survey,musin2008kissing,zong1998kissing} to show that $\mathcal{\perm}_n \subset \mathcal{K}_m(\mathcal{\perm}_n)$ for a surprisingly small $m$. 
Let us first formalize our approach by defining the \textit{Kissing number}:
\begin{definition}
For a given $m \in \mathbb{N}$, we define the \textit{Kissing number} $\kiss{m}$ as
\begin{equation}
\begin{aligned}
\label{eq:kissing}
     \kiss{m} := & \max_n \{ n\in \mathbb{N}~|~ \exists A \in \mathbb{R}^{n \times m},
     ~ \|A_{i,:}\|_2=1, \ 2\langle A_{i:},A_{j,:}\rangle \leq 1,~i\neq j \}.
\end{aligned}
\end{equation}
\end{definition}

Note that the Kissing number can be interpreted geometrically as the maximum number of points that can be distributed on an $m$-dimensional unit sphere such that the angle formed between each pair of different points is at least $\arccos(0.5)$. This property quickly establishes $\mathcal{\perm}_n \subset \mathcal{K}_m(\mathcal{\perm}_n)$ : 
\begin{proposition}\label{prop:relu}
Let $\perm \in \Perm_n$ be an arbitrary permutation matrix, and let $\sigma$, $\sigma(x)=\max(x,0)$ denote a rectified linear unit (ReLU). Then for every $m$ such that
$n\leq \kiss{m}$  there exist $\inA,\inB \in \mathbb{R}^{n\times m} $ such that 
\begin{equation} \label{eq:prop1}
   \perm =\sigma(2\inA \inB^T-1). 
\end{equation}
\end{proposition}
\begin{proof}
Let $\inA \in \mathbb{R}^{n\times m}$ be a matrix that satisfies the equalities and inequalities of \eqref{eq:kissing}, and let $\inB = \perm\inA$. Then it holds that 
\begin{align}
    \label{eq:reluRepresentation}
2\langle \inA_{i,:}, \inB_{j,:} \rangle \begin{cases}
\leq 1 & \text{ if } \perm_{i,j} \neq 1\\
=2 & \text{ otherwise}
\end{cases}.
\end{align}
Consequently 
\begin{align}
    \sigma(2\langle \inA_{i,:}, \inB_{j,:} \rangle-1)= \begin{cases}
0 & \text{ if } \perm_{i,j} \neq 1\\
1 & \text{ otherwise}
\end{cases},
\end{align}
which proves the assertion. 
\end{proof}

To determine the minimal rank $m$ that is required for representing a permutation of $n$ elements, we rely on extensive studies in the past few decades which computed either exact values or lower and upper bounds for different values of $m$ \cite{caluza2018improving}.

Using $\perm = \sigma(2VW^T-1)$ for relaxing \eqref{eq:energyMin} yields a relaxation that requires only $2mn$ instead of $n^2$ parameters, with $m<<n=\kiss{m}$.
For instance, $\kiss{24}=196560$ implies that matrices of rank $m=24$ are sufficient for representing any arbitrary permutation matrix of up to $n=196560$ elements, thus requiring $\sim4000$ times less storage memory: $2\cdot24\cdot196560$ instead of $196560^2$ parameters.
Furthermore, $\mathcal{\perm}_n \subset \mathcal{K}_m(\mathcal{\perm}_n)$ ensures that -- in stark contrast to direct low-rank factorization -- \textit{any} permutation matrix can still be represented \textit{exactly}.  
Empirically, the optimization over  parametrizations $\sigma(2VW^T-1)$ turned out to cause significant challenges, likely due to the non-convexity and non-smoothness of the problem. To alleviate this problem, we resort to a smoother version of~\eqref{eq:prop1} which can still approximate permutations to an arbitrary desired accuracy:

\begin{proposition}\label{prop:softmax}
    Let $\perm \in \Perm_n$ and $g$ denote an arbitrary permutation matrix and an arbitrary entry-wise strictly monotonically increasing function, respectively, and let $s$ denote the row-wise Softmax function 
    $s(A)_{i,j}= \frac{\exp{A_{i,j}}}{\sum_k \exp{A_{i,k}}}$.
    Then $\forall n\leq \kiss{m}$  and $\forall\epsilon>0$ there exist $\inA,\inB \in \mathbb{R}^{n\times m}$ and $\alpha>0$, such that 
    $$
    \|\perm -s\left(\alpha g(\inA \inB^T)\right)\|\leq \epsilon.
    $$
\end{proposition}
\begin{proof}
    Similar to the proof in Proposition~\ref{prop:relu}, we start by choosing $\inA$ satisfying \eqref{eq:kissing} and setting $\inB=\perm\inA$ to obtain
    $$ (\inA \inB^T)_{ij} = \langle \inA_{i,:}, \inB_{j,:} \rangle  \begin{cases}
        = 1 & \text{ if } \perm_{i,j} =1 \\
        \leq 0.5 & \text{otherwise}
    \end{cases}.$$
    Then $\forall i,j,k$ s.t. $\perm_{ij}=1$ and $k \neq j$ (i.e., $\perm_{ik}=0$) it holds that \mbox{$g(\inA \inB^T)_{ij}> g(\inA \inB^T)_{ik}$}. 
    Finally, to yield the assertion we use the Softmax property of converging to the unit vector in the limit $s(\alpha A_{i,:})\stackrel{\alpha \rightarrow \infty}{\rightarrow} e_j$ (with $j=\argmax{A_{i,:}}$), by taking $\alpha>0$ to be large enough.
\end{proof}
In practice, we use $g(x) = 2x$, in accordance with the representation in~\eqref{eq:prop1}.
We use this smoother version to validate the proposed low-rank representation for handling large matching problems in the experiments we report next.

\section{Experiments}
The following experiments validate our efficient permutation estimation method for different applications, and they confirm the ability to scale to very large problem sizes. 
First, as a proof of concept, we demonstrate our approach on the application of point cloud alignment for the two non-linearities proposed in Section \ref{sec:method} and introduce our sparse training technique.
We then valdidate the effectiveness of our approach in the context of linear assignment problems and show how to handle sparse cost matrices.
We perform further experiments in the context of generic NP-hard quadratic assignment problems, and integrate our approach into a state-of-the-art shape matching pipeline, thus providing the same level of accuracy while enabling a higher spatial resolution.

\subsection{Implementation Details} \label{sec:impl_details}

We use the PyTorch Adam optimizer \cite{kingma2015adam} with its default hyperparameters in all our experiments. 

\paragraph{Stochastic Optimization.} 
Fully benefiting from our proposed compact representation
requires the costs $E$ (or an approximation thereof) to be evaluated \textit{without ever forming the full (approximate) permutation matrix}, as this step would inherently return to necessitate $n^2$ many entries. 
To this end, we introduce the concept of \textit{stochastic optimization}, which -- for our softmax-based representation $s(2\alpha VW^T)$ arising from Proposition \ref{prop:softmax} -- is not a stochastic training in a classical sense: we propose to fix all but two entries in each row of our approximate permutation. 
More specifically, in any supervised (learning-based) scenario where it is known that the $y_i$-entry of the $i$-the row of the final permutation $P$ ought to be equal to one, each step of our optimizers merely computes the $y_i$-th and one randomly chosen ($r_i$-th entry) of each row, and computes the softmax $s$ on these two entries only, i.e., 
\begin{equation}
    P_{i,[y_i,r_i]} = s(2\alpha V_{i,:} (W_{[y_i,r_i],:})^T),
\end{equation}
while implicitly assuming $P_{i,j}=0$ for $j \notin \{y_i,r_i\}$. 

In the above, we used $W_{[y_i,r_i],:}$ to denote the $2 \times m$ matrix consisting of the $y_i$-th and the $r_i$-th row of $W$. Our stochastic approach requires the computation of 2$n$ entries per gradient descent iteration only and -- by randomly choosing the $r_i$ -- manages to still approximate the desired objective well. 

\paragraph{Normalization of $V$ and $W$.} Since propositions \ref{prop:relu} and \ref{prop:softmax} rely on row-normalized matrices, 
we explicitly enforce this constraint whenever we compute $\perm$, by using \mbox{$V_{i,:} \leftarrow \frac{1}{\|V_{i,:}\|}V_{i,:}$}, \mbox{$W_{i,:} \leftarrow \frac{1}{\|W_{i,:}\|}W_{i,:}$}.
We omit this step from the presentations below for the sake of readability

\paragraph{Softmax Temperature.} Since the values of $\langle V_{i,:}, W_{j,:}\rangle$ are bounded by one following the aforementioned normalization, the \textit{temperature} parameter $\alpha$ determines the trade-off between approximating a hard maximum (as required for accurately representing permutations, see Prop.\ref{prop:softmax}) and favorable optimization properties (i.e., meaningful gradients). We specify the schedule (constant or monotonically increasing) in each of the experiments below.

\subsection{Point Cloud Alignment}
\label{sec:pointcloud}

As a proof of concept, we demonstrate that our proposition is correct and the optimization process converges. We explore the different choices of non-linearity, starting with ReLU and continuing with Softmax, using the task of predicting a linear transformation over point clouds. 
In this task we aim to match a point cloud $X_1 \in \mathbb{R}^{n \times m}$ consisting of $n$ $m$-dimensional points, uniformly distributed on the unit hyper-sphere, to its linearly transformed and randomly permuted version $X_2 \in \mathbb{R}^{n \times m}$.
To obtain $X_2$ we multiply $X_1$ by a randomly drawn matrix $\Theta_\text{GT} \in \mathbb{R}^{m\times m}$
and apply a random permutation.
Then, we optimize over the estimated transformation matrix $ \Theta$ which in this experiment defines our permutation matrix $P(\Theta)$:
\begin{equation}
    P(\Theta) = \sigma\left(2 \inA \inB(\Theta)^T - 1 \right).
\end{equation}
Note that in this case, our representation in \eqref{eq:prop1} is fully parameterized by $\Theta$, with $\inA=X_1$ and $\inB(\Theta) = X_2\Theta$, and $V$ and $W$ are row-wise normalized in each iteration.
Here $P(\Theta)$ is equal to the correct permutation if the matrix $\Theta$ correctly aligns the point clouds, i.e. minimizes the angle between two corresponding points in $\inA$ and $\inB(\Theta)$ while maximizing the angles between non-corresponding points. %

We solve for the permutation by performing $20000$ minimizing steps with a learning rate set to $0.01$ over the negative log-likelihood loss
\begin{equation}\label{eq:pointcloud}
    \hat{\Theta}=\argmin_{\Theta}~-\frac{1}{n}\sum_{i=1}^n \log\bigl(P(\Theta)_{i,y_i}\bigr),
\end{equation}
where $y_i$ is the index of the point in $X_2$ which corresponds to the $i^\text{th}$ point in $X_1$.
We experiment with different numbers of points $n$, each time choosing the dimension $m$ to be just big enough to satisfy the Kissing number constraint from Proposition~\ref{prop:relu}, i.e., $\kiss{m} \geq n > \kiss{m-1}$.

To check that we were able to find the correct transformation matrix $\Theta$ -- and therefore the correct permutation matrix $\perm$ -- through optimization, we verify that the nearest neighbor (closest point) for each row $i$ in $V$ is located in row $j$ of matrix $W$ that satisfies $\perm_{i,j}=1$. We find that this is indeed the case in all experiments with different number of points 
$n\in\{10,100,1000,10000\}$, thus establishing that we could reach the correct representation through optimization.
We achieve equally good results when replacing the point-wise non-linearity ReLU with Softmax $P(\Theta) = s\left(2 \alpha \inA \inB(\Theta)^T\right)$.

Due to the quadratically growing size of the permutation matrix with an increasing number of points, we further propose to optimize for the permutation matrix stochastically, as described in Section \ref{sec:impl_details}. We ran experiments with similar settings as above, wherein we gradually increased the value of the temperature parameter $\alpha$ linearly during optimization from $\alpha=5\cdot10^{-5}$ to $\alpha=1000$. In these experiments, we again found that each point was paired with its corresponding nearest neighbor, while reducing the memory consumption, as shown in Fig.\;\ref{fig:point_registration}.

\subsection{Linear Assignment Problems}
We next validate our method on balanced Linear Assignment Problems (LAPs), which typically involve assigning a set of agents to an equally sized set of tasks, e.g., when optimizing the allocation of resources.
We show results on a collection of regularized LAPs in the form
\begin{align}\label{eq:LAP_defenition}
    \argmin_{\inA,\inB} \quad \underbrace{\textbf{tr}(A \cdot \perm(V,W))}_{\text{LAP term}} + \underbrace{\mu(\perm(V,W))}_{\text{regularizer}},
\end{align}
where $\perm(V,W) = s\left(2 \alpha \inA \inB^T\right)$ is a permutation and $A \in \mathbb{R}^{n \times n}$ is some given similarity matrix.
While the Softmax non-linearity ensures all rows sum to one, $\mu(\perm(V,W))$ is a regularization term enforcing columns summing to one as well, to satisfy the permutation constraints:

\begin{equation}\label{eq:perm_constraint} \textstyle
\mu(\perm) = %
\sum_j  \left( \sum_i \perm_{ij} - 1 \right)^2.
\end{equation}
Due to the row-wise Softmax all rows already sum to one but we incentive the columns to sum to one as well, as is necessary for permutations. 

\paragraph*{Dense Matrices.} We evaluate on a set of LAPs based on descriptor similarities of 3D shapes from the FAUST dataset of human scans~\cite{bogo2014faust}, with $n$ randomly chosen vertices per object~\cite{gao2021multi}.
Let $D_X, D_Y \in \mathbb{R}^{n \times k}$ be two k-dimensional point-wise descriptors of the shapes $X, Y$ corresponding to $n$ points. 
We use the SHOT \cite{salti2014shot} ($\mathbb{R}^{n \times 352}$) and the heat kernel signature \cite{sun2009hks} ($\mathbb{R}^{n \times 101}$) with their default parameters as descriptors and stack them together to comprise $D_{\cdot} \in \mathbb{R}^{n \times 453}$ in total, then $A = D_X \cdot D_Y^\top$.
Solving an LAP with this type of similariy matrix $A$ is used e.g. in \cite{vestner2017kernel} as the initialization strategy. 

We generate $100$ problem instances by pairing each of the $100$ shapes in FAUST with a random second shape to get the pair $X, Y$ and evaluate the relative error of the energy (restricted to solutions that were valid permutations), and the average Hamming distance to the next valid permutation (namely, the number of rows or columns that violate the permutation constraint). 
We ran the experiments with $n=100, m=30, \alpha=20$ and used a greedy heuristic to generate valid permutations from the results violating the permutation constraint (iteratively projecting the maximum value of the permutation to one, and the rest of the corresponding row and column to zero).
Out of the $100$ instances, $53$ lead to valid permutations without the heuristic.
The average relative error of immediately valid permutation is $1.8\%$ and after pseudo-projection of all instances it is $2.0\%$.
Due to the Softmax, every matrix has $100$ non-zero entries that are all nearly equal to one.
On average, the Hamming distance of invalid permutations to the next valid one is $1.38$ ($1.4\%$ of the problem size) which means in most cases one would have a valid permutation after only adjusting one entry.

\paragraph*{Sparse Matrices.} 
Given a matrix $A$ , that is sparsely populated and only contains non-zero entries in a subset $S=\{ (i,j)| A_{i,j} \neq 0\}$, we compute and optimize the permutation matrix sparsely in $(i,j) \in S$ 
by calculating the matrix factorization only at the required entries, similar to Section \ref{sec:impl_details}, but without restricting the number of entries per row of $P$ to two. Also we take into account random entries $(q,r) \notin S$.
We ran experiments for $A$ with a matrix density of  $|S|=0.01n^2$ for $n=1000,5000 \text{ and } 10000$ and $m=20$ with increasing $\alpha$ from $1$ to $20$ iteratively and measure a
Hamming distance of at most $0.28 \%$ of the problem size. To get a valid permutation matrix, we used the same heuristic as in the dense case  and measured a relative error below $7.8\%$, compared to the Hungarian algorithm. Also, we could measure a memory reduction by over $65\%$.

\begin{figure}[!tbp]
  \centering
  \begin{minipage}[t]{0.33\textwidth}
  \centering
    \definecolor{colornet}{rgb}{0.00000, 0.00000, 0.60000}%
\definecolor{colorfull}{rgb}{0.00000, 0.60000, 0.00000}%
\definecolor{color3}{rgb}{0.60000, 0.00000, 0.00000}%
\definecolor{color4}{rgb}{0.00000, 0.60000, 0.60000}%
\definecolor{color5}{rgb}{0.60000, 0.00000, 0.60000}%

\begin{tikzpicture}

\begin{axis}[%
width=.7\linewidth,
height=.5\linewidth,
scale only axis,
xmin=0,
ymode=log,
grid=both,
grid style={line width=.1pt, draw=gray!50},
xlabel style={font=\color{white!15!black}},
x tick label style = {font=\footnotesize},
y tick label style = {font=\footnotesize},
x label style={at={(axis description cs:0.4,0.04)}, anchor=north},
xlabel={$n$},
ylabel = { \small Mem. Cons. in GB},
y label style={at={(axis description cs:0.1,.5)},rotate=0,anchor=south},
axis background/.style={fill=white},
legend style={nodes={scale=0.8, transform shape}},
        legend cell align=left,
        smooth,
        fill=white, fill opacity=0.6, text opacity=1
]

\addplot[ mark=*, mark options={draw=colornet, fill=colornet}, mark size= 1pt, draw=colornet] table[row sep=crcr]{%
x	y\\
100 0.002048 \\ 
1000 0.002048 \\ 
2000 0.002048 \\ 
3000 0.002048 \\ 
4000 0.004096 \\ 
5000 0.004096 \\ 
6000 0.006144 \\ 
7000 0.006144 \\ 
8000 0.008192 \\ 
9000 0.008192 \\ 
10000 0.01024 \\ 
11000 0.01024 \\ 
12000 0.01024 \\ 
13000 0.01024 \\ 
14000 0.022528 \\ 
15000 0.022528 \\ 
16000 0.022528 \\ 
17000 0.022528 \\ 
18000 0.022528 \\ 
19000 0.022528 \\ 
20000 0.024576 \\ 
21000 0.024576 \\ 
22000 0.024576 \\ 
23000 0.024576 \\ 
24000 0.024576 \\ 
25000 0.024576 \\ 
26000 0.024576 \\ 
27000 0.024576 \\ 
28000 0.045056 \\ 
29000 0.045056 \\ 
30000 0.045056 \\ 
31000 0.045056 \\ 
32000 0.045056 \\ 
33000 0.045056 \\ 
34000 0.045056 \\ 
35000 0.045056 \\ 
36000 0.045056 \\ 
37000 0.045056 \\ 
38000 0.047104 \\ 
39000 0.047104 \\ 
40000 0.047104 \\ 
41000 0.047104 \\ 
42000 0.047104 \\ 
43000 0.047104 \\ 
44000 0.047104 \\ 
45000 0.047104 \\ 
46000 0.047104 \\ 
47000 0.047104 \\ 
48000 0.067584 \\ 
49000 0.067584 \\ 
50000 0.067584 \\ 
51000 0.067584 \\ 
52000 0.067584 \\ 
53000 0.067584 \\ 
54000 0.067584 \\ 
55000 0.067584 \\ 
56000 0.067584 \\ 
57000 0.067584 \\ 
58000 0.067584 \\ 
59000 0.067584 \\ 
60000 0.067584 \\ 
61000 0.067584 \\ 
62000 0.067584 \\ 
63000 0.067584 \\ 
64000 0.067584 \\ 
65000 0.067584 \\ 
66000 0.069632 \\ 
67000 0.069632 \\ 
68000 0.069632 \\ 
69000 0.069632 \\ 
70000 0.069632 \\ 
71000 0.069632 \\ 
72000 0.069632 \\ 
73000 0.069632 \\ 
74000 0.069632 \\ 
75000 0.069632 \\ 
76000 0.110592 \\ 
77000 0.110592 \\ 
78000 0.11264 \\ 
79000 0.11264 \\ 
80000 0.110592 \\ 
81000 0.110592 \\ 
82000 0.110592 \\ 
83000 0.11264 \\ 
84000 0.11264 \\ 
85000 0.11264 \\ 
86000 0.110592 \\ 
87000 0.110592 \\ 
88000 0.11264 \\ 
89000 0.11264 \\ 
90000 0.11264 \\ 
91000 0.11264 \\ 
92000 0.11264 \\ 
93000 0.11264 \\ 
94000 0.11264 \\ 
95000 0.11264 \\ 
96000 0.11264 \\ 
97000 0.11264 \\ 
98000 0.11264 \\ 
99000 0.11264 \\ 
100000 0.114688 \\ 
};
\addplot[ mark=*, mark options={draw=colorfull, fill=colorfull}, mark size= 1pt, draw=colorfull] table[row sep=crcr]{%
x	y\\
100 0.002048 \\ 
1000 0.022528 \\ 
2000 0.083968 \\ 
3000 0.186368 \\ 
4000 0.321536 \\ 
5000 0.495616 \\ 
6000 0.710656 \\ 
7000 0.946 \\ 
8000 1.236 \\ 
9000 1.556 \\ 
10000 1.918 \\ 
11000 2.318 \\ 
12000 2.758 \\ 
13000 3.24 \\ 
14000 3.762 \\ 
15000 4.322 \\ 
16000 4.912 \\ 
17000 5.542 \\ 
18000 6.202 \\ 
19000 6.912 \\ 
20000 7.652 \\ 
21000 8.442 \\ 
22000 9.262 \\ 
23000 10.112 \\ 
24000 11.012 \\ 
25000 11.952 \\ 
26000 12.922 \\ 
27000 13.932 \\ 
28000 14.982 \\ 
29000 16.072 \\ 
30000 17.192 \\ 
31000 18.352 \\ 
32000 19.562 \\ 
};

\addlegendentry{sparse}
\addlegendentry{dense}
legend style={at={(0,-100.5)},anchor=west}
\end{axis}

\end{tikzpicture}%
    \caption{Memory consumption for point cloud alignment
        \label{fig:point_registration}}
  \end{minipage}
  \hfill
  \begin{minipage}[t]{0.66\textwidth}
  \centering
    \definecolor{mycolor1}{rgb}{0.00000, 0.60000, 0.00000}%
\begin{tikzpicture}

\begin{axis}[%
width=.35\linewidth,
height=.25\linewidth,
at={(0.758in,0.481in)},
scale only axis,
xmin=0,
ymin=0,
ymax=0.6,
grid=both,
grid style={line width=.1pt, draw=gray!50},
xtick = {12, 40, 64, 100, 150},
xlabel style={font=\color{white!15!black}},
x tick label style = {font=\footnotesize},
y tick label style = {font=\footnotesize},
x label style={at={(axis description cs:0.5,0.04)}, anchor=north},
xlabel={\small Instance Size},
ylabel style={font=\color{white!15!black}},
ylabel={ \small Relative Error},
y label style={at={(axis description cs:0.2,.5)},rotate=0,anchor=south},
axis background/.style={fill=white},
]
\addplot[only marks, mark=*, mark options={fill=mycolor1}, fill, mark size=2.0pt, draw=black!30!mycolor1] table{%
x  y
90 0.207497397409747
15 0.0420654414048163
16 0
150 -0.0503297637715503
80 0.20404208572891
64 0.0280840636059157
50 0.0113539368366805
25 0.00194515759442492
20 0.166231505657093
19 0.294498176562819
30 0.0120851563333689
25 0.00213675213675214
20 0.021011673151751
32 0
16 0.0588235294117647
16 0
20 0
12 0.558626465661642
40 0.0160441372312766
60 0.0310649598838716
26 0.000952870855276311
100 -0.0265283200513451
12 0.077553802522113
22 0.191033138401559
50 0.0723802563078139
30 0.156247936285711
15 0.568210262828536
26 0.00985045342373164
30 0.00555192684519922
40 0.0862417359636474
60 0.0388155108045057
16 0.00745341614906832
16 0.142857142857143
20 0.334323292320747
60 0.0101755302281333
18 0.109517601043025
70 0.00908957026302613
12 0.0207612456747405
15 0.487873888439774
15 0.00869565217391304
16 0.230769230769231
100 -0.0695809587663426
14 0.029585798816568
16 0
80 0.0346248144414911
30 0.0204128092275004
27 0.0187237294612151
15 0.018618596959409
24 0.00114678899082569
40 0
16 0.0258064516129032
20 0.356751824817518
17 0.0482420111750018
100 -0.0281759698485931
16 0
100 -0.0180573778252695
20 0.00751227968795146
22 0.0161290322580645
12 0.0978032473734479
80 0.00804123304172673
100 -0.0574466125445718
16 0
100 0.0440798661724477
32 -0.797840269966254
81 -0.0258027648959318
18 0.106505676698504
30 -0.127433821920805
50 0.17296677364234
20 0.0378773017646507
36 0.0485032970794176
26 0.00379040752638225
12 0.0989769000427777
100 -0.0747604605598647
26 0.00392441579237837
26 0.00545699518996546
26 0.00449428470867033
18 0.000373273609555804
60 0.191564340799203
150 0.00569576291977575
90 0.0078196489476749
12 0.285765507350305
30 0.171563319787959
80 -0.013309618683297
22 0.148853729415563
20 0.0273561755884759
15 0.0852561595619867
90 -0.0223830214482317
};

\end{axis}
\end{tikzpicture}%
    \definecolor{mycolor1}{rgb}{0.00000, 0.60000, 0.00000}%
\begin{tikzpicture}

\begin{axis}[%
width=.35\linewidth,
height=.25\linewidth,
at={(0.758in,0.481in)},
scale only axis,
xmin=0,
ymin=0,
grid=both,
grid style={line width=.1pt, draw=gray!50},
xtick = {12, 40, 64, 100, 150},
xlabel style={font=\color{white!15!black}},
x tick label style = {font=\footnotesize},
y tick label style = {font=\footnotesize},
x label style={at={(axis description cs:0.5,0.04)}, anchor=north},
xlabel={ \small Instance Size},
ylabel style={font=\color{white!15!black}},
ylabel={ \small Runtime in Sec.},
y label style={at={(axis description cs:0.2,.5)},rotate=0,anchor=south},
axis background/.style={fill=white},
]
\addplot[only marks, mark=*, mark options={fill=mycolor1}, fill, mark size=2.0pt, draw=black!30!mycolor1] table{%
x	y
90 1442.46585965157
15 21.6270885467529
16 21.7824952602386
150 10815.5356087685
80 960.677081346512
64 407.166689634323
50 145.622968673706
25 24.3189704418182
20 22.5130889415741
19 22.3655817508698
30 30.2580587863922
25 24.3296601772308
20 22.8182778358459
32 31.8976752758026
16 21.8436396121979
16 21.8782675266266
20 22.6432144641876
12 20.4518513679504
40 61.7051064968109
60 318.239206790924
26 24.6964631080627
100 2293.78922748566
12 20.6244349479675
22 23.2520067691803
50 146.156038284302
30 29.3514175415039
15 22.0652923583984
26 25.6351211071014
30 31.4353511333466
40 61.2287142276764
60 318.764275074005
16 22.0553381443024
16 22.0460960865021
20 22.8959896564484
60 318.930606126785
18 22.5233271121979
70 574.709892749786
12 21.0991325378418
15 22.3346455097198
15 22.410026550293
16 22.4794127941132
100 2291.95836734772
14 22.065258026123
16 22.5012216567993
80 965.43284034729
30 31.1895346641541
27 25.1184298992157
15 21.5565037727356
24 23.6183428764343
40 60.645644903183
16 21.6710629463196
20 22.4000844955444
17 22.0619440078735
100 2277.51474499702
16 22.7776539325714
100 2281.36911416054
20 22.8527336120605
22 23.5169150829315
12 21.3596692085266
80 965.125685453415
100 2263.8715775013
16 22.5454783439636
100 2272.81606841087
32 30.9148547649384
81 1010.84697580338
18 23.2401657104492
30 32.434047460556
50 146.558495521545
20 23.3162605762482
36 44.8751046657562
26 25.4121587276459
12 21.0417692661285
100 2270.44675922394
26 24.6044888496399
26 24.4784572124481
26 25.0829477310181
18 22.388222694397
60 317.829624176025
150 10221.4431014061
90 1533.61711144447
12 21.1573421955109
30 31.785254240036
80 968.663019895554
22 24.4189376831055
20 23.6538648605347
15 22.490008354187
90 1495.94621038437
};

\end{axis}
\end{tikzpicture}%
    \caption{Relative error and runtime on QAPLIB dataset\label{fig:qaplib:results}}
  \end{minipage}
\end{figure}

\subsection{Quadratic Assignment Problems}\label{subsec:qap}

Quadratic Assignment Problems (QAPs) is a broadly employed mathematical tool for many real-life problems in operations research such as circuit design and shape matching. We demonstrate the application of our approach to non-convex QAPs of the form
\begin{equation} \label{eq:qap}
   \argmin_{\inA,\inB} p(V,W)^T A\ p(V,W)
\end{equation}
where $p(V,W) = \text{vec}(s\left(2\alpha \inA \inB^T\right))$  is the vectorized version of the permutation and $A \in \mathbb{R}^{n^2 \times n^2}$ is a cost matrix. $\inA$ and $\inB$ are normalized.
 The permutation matrix was optimized in a convex-concave manner, by optimizing the objective function
 \begin{equation} \label{eq:qap-id}
   \argmin_{\inA,\inB} p(V,W)^T (A-\beta I)\ p(V,W) + \mu(\perm(V,W))
\end{equation}
with $\beta$ being iteratively increased from $-\Vert A \Vert_2$ to $\Vert A \Vert_2$ and with
 $\mu(P(V,W))$ being the same permutation constraint regularizer as in \eqref{eq:perm_constraint}.

We show results on the QAPLIB \cite{burkard1997qaplib} library of quadratic assignment problems of real-world applications which range between $n=12$ and $n=256$ and we choose $m=\text{ceil}(\frac{n}{3})$.
The problems in the dataset are meant to be challenging and optimal solutions for some of the larger problems are not provided because they are unknown. 
Thus, we report the gap to optimality (when known) of our solution
and consider a solution to be good if it falls within $10\%$ of the optimum.
We report the relative error and runtime in Fig.~\ref{fig:qaplib:results}.
In $75$ out of $87$ instances the result was a proper permutation matrix.

\subsection{Shape Matching \label{sec:shape_matching}}
Finally, we further assess the effectiveness of our approach for the the application of non-rigid shape matching, a common task in computer graphics and computer vision.
To this end, we incorporate our permutation matrix representation approach into
the state-of-the-art shape-matching approach by Marin \etal \cite{marin2020correspondence},
which learns the point correspondences using two consecutive networks $N_\theta$ and $G_\theta$, predicting shape embeddings and probe functions, respectively. We propose to replace the calculation of the permutation matrix based on the output of the first network $N_\theta$  by $s\left(\alpha \inA \inB^T \right)$, with $\alpha = 40$. 
The network transforms the vertices of 3D objects $X_x$ and $X_y$ into embeddings $\phi_x = N_\theta(X_x)$ and $\phi_y = N_\theta(X_y)$, which are used to compute 
$
            V = \phi_x (\phi_x^{\dagger} \perm_{gt} \phi_y)
$
 and $W = \phi_y$. $V$ here replaces a transformed embedding.
The network is trained on the modified loss function
\begin{align}
    \min_\theta \sum_l \|s\left(2\alpha \inA \inB^T \right)^l X_y^l - \perm_{gt}^l X_{y}^l\|_2^2 \label{eq:shape_matching_loss}
\end{align}
for a given ground truth permutation $\perm_{gt}$, and $V$ and $W$ being normalized row-wise. 
Similar to Marin \etal, we train the networks over $1600$ epochs on $10000$ shapes of the SURREAL dataset \cite{varol17_surreal} and evaluate our experiments on $100$ noisy and noise-free objects of different shapes and poses of the FAUST dataset \cite{bogo2014faust}, that are provided by \cite{marin2020correspondence} in \cite{maringithub}. 

We follow the evaluation of Marin \etal~\cite{marin2020correspondence} and calculate the geodesic distance between the ground truth matching and the predicted matching $match_1 = \mathcal{N}(\phi_x C_1^T ,\phi_y )$
for $ C_1 = ((\phi_y^{\dagger}G_\theta(X_y))^T)^\dagger (\phi_x^{\dagger}G_\theta(X_x))^T$ whereby $\mathcal{N}$ is the nearest neighbor.
In the following, we refer to the measured geodesic distance as $\errorOne$.
A second error ($\errorTwo$) which only concerns the first network's predictions, is measured by the geodesic distance towards $match_{2} = \mathcal{N}(\phi_x,\phi_y C)$
for $C = \phi_y^{\dagger}\perm_{gt}\phi_x$, which is, again, calculated following Marin \etal~\cite{marin2020correspondence}. %
The results of our experiments are reported in Table \ref{tab:noise}, showing the average geodesic errors (over $10$ runs for each experiment) for the approach presented in \cite{marin2020correspondence} and our method. The table reveals improved results compared to~\cite{marin2020correspondence}.

\begin{table}[]
\caption{Geodesic errors and standard deviation (\textit{std}) for noise-free and noisy data by Marin \etal \cite{marin2020correspondence} and our approach \\}
\label{tab:noise}
  \centering
\begin{tabular}{llllll}
\toprule 
& $\errorOne$ & \textit{std} ($\errorOne$)         & $\errorTwo$  & \textit{std} ($\errorTwo$)& \\ \midrule
\cite{marin2020correspondence} &$ 0.051 $&$17.4 \times 10^{-4}$ & $0.029$ &$3.5\times 10^{-4}$ & \\
ours & $0.047$ &$26.9\times 10^{-4}$& $0.026$ &$29.2\times 10^{-4}$ &
\multirow{-2}{*}{noisy} \\ 
\midrule
\cite{marin2020correspondence}   &$ 0.043$ &$16.3\times 10^{-4}$  &$ 0.022$ &$3.5\times 10^{-4}$ &  \\
ours &$ 0.041 $&$8.1\times 10^{-4}$ &$ 0.019 $&$3.7\times 10^{-4}$& 
\multirow{-2}{*}{noise-free} \\ 
\bottomrule
\end{tabular}
\end{table}

\paragraph*{Stochastic Training. }
Given that the explicit calculation of the permutation matrix in \eqref{eq:shape_matching_loss} is memory-intensive for a large number of vertices, we employ stochastic training to avoid the need for computing the full permutation matrix. As we describe in Sec.~\ref{sec:impl_details} we only calculate the loss over a few entries where the final permutation ought to be equal to one and on $k$ (here $k$ can be $\geq 1$) randomly chosen entries of each row of $\perm$ in each iteration. %
This approach reduces the memory requirement and gives us the possibility to train with larger shapes consisting of more vertices.
In our experiments, we applied the stochastic training technique on the SURREAL dataset, and then evaluated the performance on FAUST by measuring the error rates for varying values of $k$, as depicted in Fig.\;\ref{fig:sparse_faust}.
We observed a small relative increase of less than $17\%$ in $\errorTwo$, and also a small effect on $\errorOne$, but with a less clear tendency as one could see for $\errorTwo$. For $\errorOne$ we measured an average standard deviation of $2.25 \times 10^{-3}$ and for $\errorTwo$ of $3.3 \times 10^{-4}$.
Two noise-free examples of correspondences, visualized for full training and for stochastic training with $k=1$, are shown in Fig.\;\ref{fig:faust}, with the reference image on the left and the corresponding shapes on the right.

 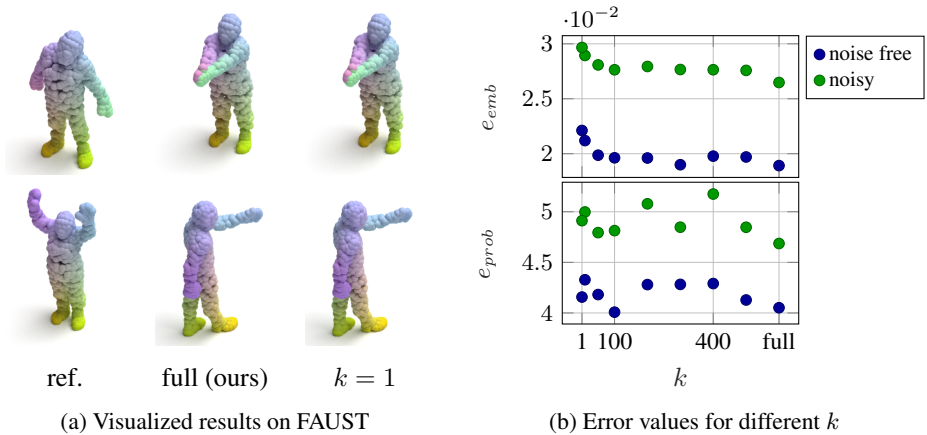
\begin{figure}
    \centering
     \begin{subfigure}[b]{0.45\textwidth}
     \centering
     \center{
\begin{tabular}{m{0.25\textwidth} m{0.25\textwidth} m{0.25\textwidth}}
         \includegraphics[width=0.25\textwidth,trim={9cm 4cm 9cm 3.8cm},clip] {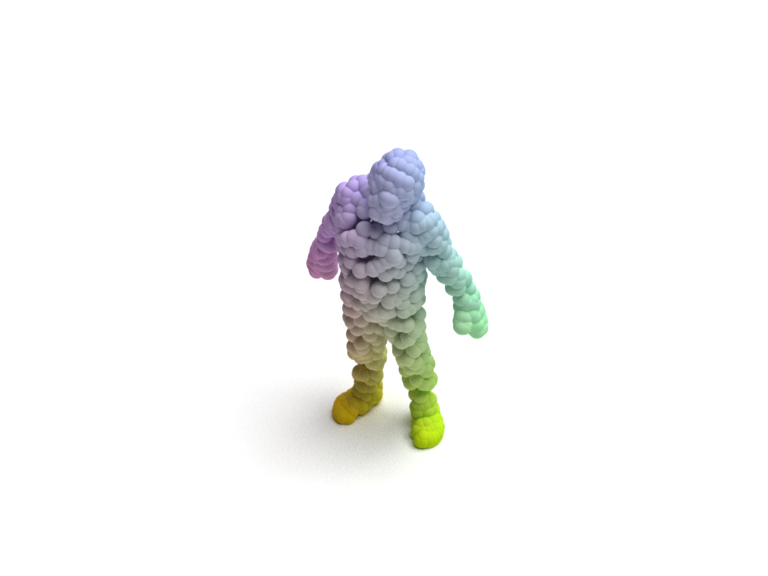}&
         \includegraphics[width=0.25\textwidth,trim={9cm 4cm 9cm 3.8cm},clip]{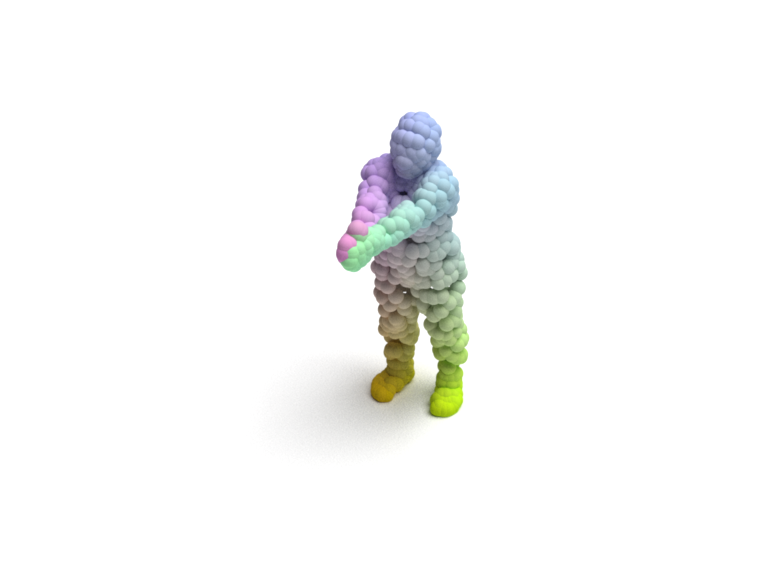}&
         \includegraphics[width=0.25\textwidth,trim={9cm 4cm 9cm 3.8cm},clip]{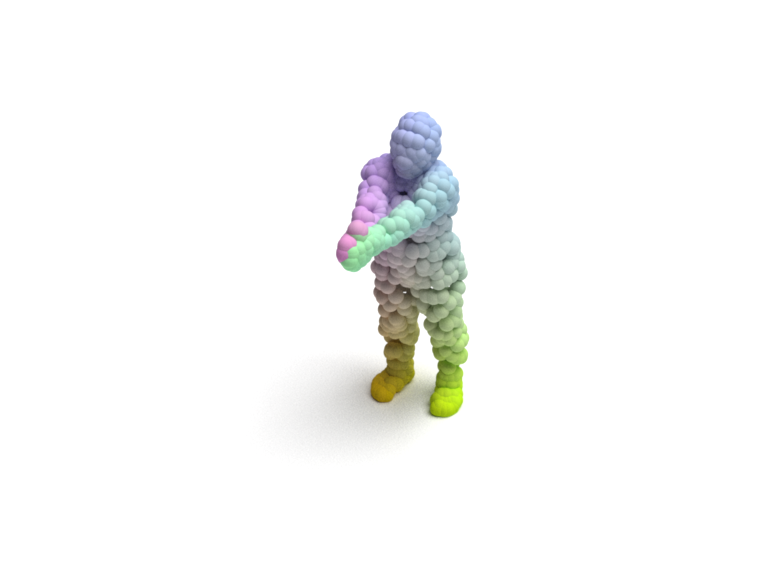}
         \\

         \includegraphics[width=0.25\textwidth,trim={9cm 4cm 9cm 3.8cm},clip]{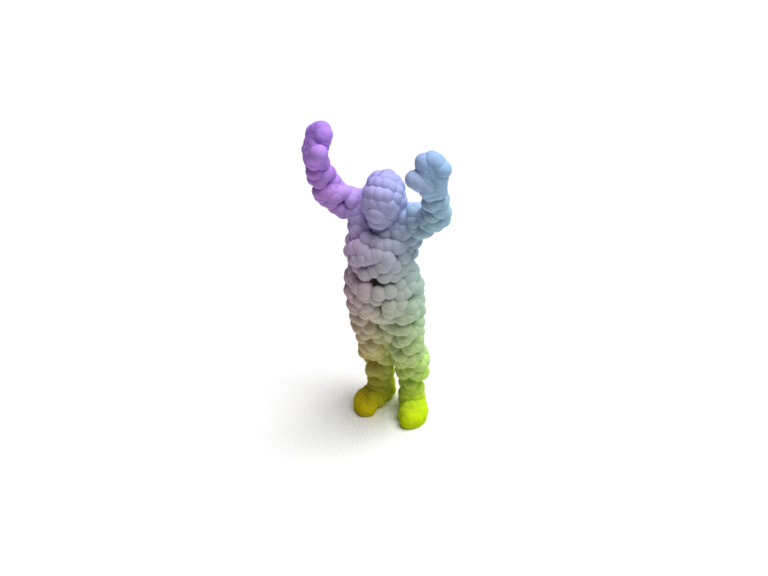}&
         \includegraphics[width=0.25\textwidth,trim={9cm 4cm 9cm 3.8cm},clip]{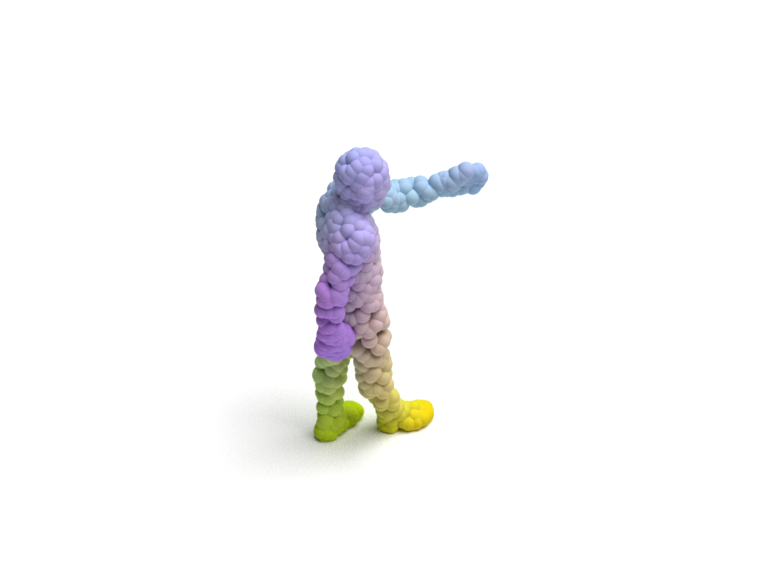}&
         \includegraphics[width=0.25\textwidth,trim={9cm 4cm 9cm 3.8cm},clip]{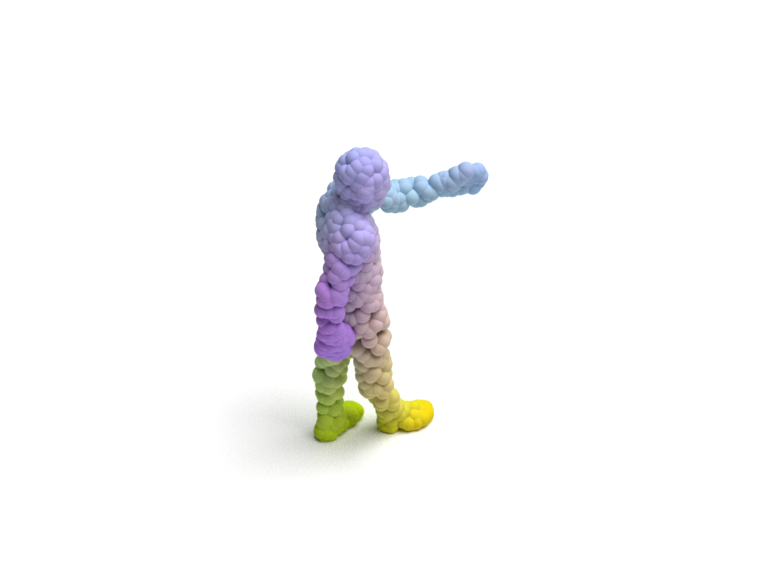}
         \\
         \center{ref.}&\center{full (ours)}&\center{$k=1$}
\end{tabular}
}
        \caption{Visualized results on FAUST \label{fig:faust}}   
        
     \end{subfigure}
     \begin{subfigure}[b]{0.45\textwidth}
     \centering
    \definecolor{colornet}{rgb}{0.00000, 0.00000, 0.60000}%
\definecolor{colorsparsegsd}{rgb}{0.00000, 0.60000, 0.00000}%
\definecolor{colorfull}{rgb}{0.60000, 0.00000, 0.00000}%
\definecolor{color4}{rgb}{0.00000, 0.60000, 0.60000}%
\definecolor{color5}{rgb}{0.60000, 0.00000, 0.60000}%

\begin{tikzpicture}

\begin{axis}[%
xmode = linear,
width=.5\linewidth,
height=.3\linewidth,
at={(0.758in,0.481in)},
scale only axis,
grid=both,
grid style={line width=.1pt, draw=gray!50},
xtick = {1, 100, 400, 600},
xlabel style={font=\color{white!15!black}},
x tick label style = {font=\footnotesize},
xticklabels =  {1, 100, 400, full},
y tick label style = {font=\footnotesize},
x label style={at={(axis description cs:0.5,0.04)}, anchor=north},
xlabel={ $k$},
ylabel style={font=\color{white!15!black}},
ylabel={\small  $\errorOne$},
y label style={at={(axis description cs:0.15,.5)},rotate=0,anchor=south},
axis background/.style={fill=white},
legend style={nodes={scale=0.8, transform shape}},
        legend pos=outer north east,
        legend cell align=left,
        smooth
]

\addplot[only marks, mark=*, mark options={fill=colornet}, mark size= 2pt, draw=black!30!colornet] table[row sep=crcr]{%
x	y\\
600 0.040515 \\ 
300	0.042821	\\
200	0.042801	\\
100	0.040077	\\
500	0.041281	\\
400	0.042905	\\
50	0.041818	\\
10	0.043280	\\
1	0.041571	\\
};

\addplot[only marks, mark=*, mark options={fill=colorsparsegsd}, mark size= 2pt, draw=black!30!colorsparsegsd] table[row sep=crcr]{%
x	y\\
500	0.048469	\\
100	0.048135	\\
400	0.051750	\\
300	0.048475	\\
200	0.050786	\\
50	0.047938	\\
10	0.049977	\\
1	0.049119	\\
 600 0.046862 \\ 
};

\end{axis}

\begin{axis}[%
xmode = linear,
yshift=1.95cm,
width=.5\linewidth,
height=.3\linewidth,
at={(0.758in,0.481in)},
scale only axis,
grid=both,
grid style={line width=.1pt, draw=gray!50},
xtick = {1, 100, 400, 600},
xlabel style={font=\color{white!15!black}},
x tick label style = {font=\footnotesize},
xticklabels =  {},
y tick label style = {font=\footnotesize},
x label style={at={(axis description cs:0.5,0.04)}, anchor=north},
ylabel style={font=\color{white!15!black}},
ylabel={\small  $\errorTwo$},
y label style={at={(axis description cs:0.15,.5)},rotate=0,anchor=south},
axis background/.style={fill=white},
legend style={nodes={scale=0.8, transform shape}},
        legend pos=outer north east,
        legend cell align=left,
        smooth
]

\addplot[only marks, mark=*, mark options={fill=colornet}, mark size= 2pt, draw=black!30!colornet] table[row sep=crcr]{%
x	y\\
300		0.018998	\\
200		0.019609	\\
100		0.019628	\\
500		0.019702	\\
400	    0.019778	\\
50		0.019866	\\
10		0.021193	\\
1		0.022113	\\
600     0.018917 \\ 
};

\addplot[only marks, mark=*, mark options={fill=colorsparsegsd}, mark size= 2pt, draw=black!30!colorsparsegsd] table[row sep=crcr]{%
x	y\\
500	0.027581	\\
100	0.027640	\\
400	0.027643	\\
300	0.027661	\\
200	0.027938	\\
50	0.028081	\\
10	0.028941	\\
1	0.029667	\\
600 0.026479 \\ 
};

\addlegendentry{noise free}
\addlegendentry{noisy}
legend style={at={(0,-100.5)},anchor=west}
\end{axis}

\end{tikzpicture}%
        \caption{Error values for different $k$ \label{fig:sparse_faust}}
     \end{subfigure}
    \caption{Visualized matching results (a) and error values (b) for the FAUST dataset for different levels of sparseness $k$ during stochastic training}
    
\end{figure}

To evaluate the impact on the error and memory consumption when dealing with objects of larger size (consisting of more vertices), we ran further experiments using data from the TOSCA dataset \cite{bronstein2008numerical}.
We trained for $400$ epochs on the objects of the classes \textit{victoria} and \textit{michael} ($32$ objects in total) where up to $20000$ vertices were sampled. These experiments revealed
a reduction of memory consumption
as shown in Fig.\;\ref{fig:sparse:results}.
We evaluated the training on the class $david$ of TOSCA and reported a relative memory-error trade-off for up to $20000$ samples of each object
in Fig.~\ref{fig:memory_tradeoff}.
The graph indicates a correlation between higher memory usage and lower error values for $\errorTwo$. 
The trends observed in the memory-error trade-off for $\errorTwo$ are generally applicable to $\errorOne$ as well, although with some noticeable outliers.

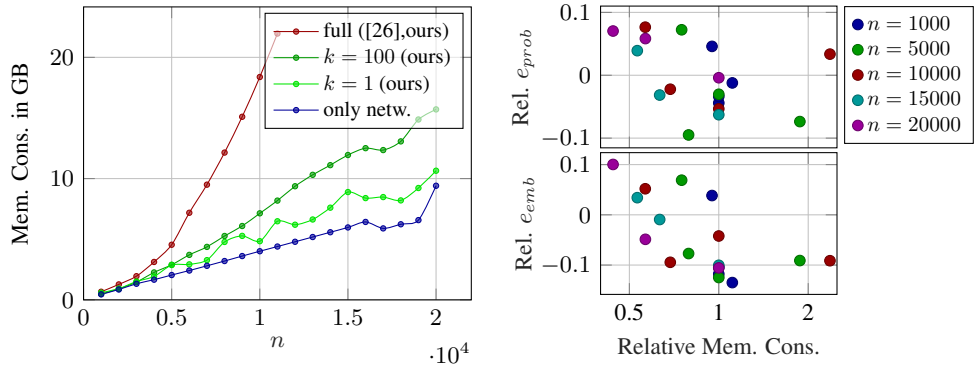
\begin{figure}
    \centering
     \begin{subfigure}[t]{0.45\textwidth}
     \centering
    \definecolor{colornet}{rgb}{0.00000, 0.00000, 0.60000}%
\definecolor{colorfull}{rgb}{0.60000, 0.00000, 0.00000}%
\definecolor{colorsparsegsd}{rgb}{0.00000, 0.60000, 0.00000}%
\definecolor{colorsparsegsdo}{rgb}{0.00000, 0.90000, 0.00000}%
\definecolor{color4}{rgb}{0.00000, 0.60000, 0.60000}%
\definecolor{color5}{rgb}{0.60000, 0.00000, 0.60000}%

\begin{tikzpicture}

\begin{axis}[%
width=.82\linewidth,
height=.62\linewidth,
at={(0.758in,0.481in)},
scale only axis,
xmin=0,
ymin=0,
grid=both,
grid style={line width=.1pt, draw=gray!50},
xlabel style={font=\color{white!15!black}},
x tick label style = {font=\footnotesize},
y tick label style = {font=\footnotesize},
x label style={at={(axis description cs:0.5,0.04)}, anchor=north},
xlabel={\small $n$},
ylabel={\small Mem. Cons. in GB},
y label style={at={(axis description cs:0.12,.5)},rotate=0,anchor=south},
legend style={nodes={scale=0.8, transform shape}},
        legend cell align=left,
        smooth,
        fill=white, fill opacity=0.6, text opacity=1
]

\addplot[ mark=*, mark options={draw=colorfull, fill=colorfull}, mark size= 1pt, draw=colorfull] table[row sep=crcr]{%
x	y\\
 1000 0.6696960000000001 \\ 
 2000 1.274 \\ 
 3000 1.952 \\ 
 4000 3.126 \\ 
 5000 4.546 \\ 
 6000 7.182 \\ 
 7000 9.5 \\ 
 8000 12.146 \\ 
 9000 15.098 \\ 
10000 18.378 \\ 
11000 21.962 \\ 
};

\addplot[ mark=*, mark options={draw=colorsparsegsd, fill=colorsparsegsd}, mark size= 1pt, draw=colorsparsegsd] table[row sep=crcr]{%
x	y\\
 1000 0.48128 \\ 
 2000 0.9 \\ 
 3000 1.466 \\ 
 4000 2.262 \\ 
 5000 2.894 \\ 
 6000 3.71 \\ 
 7000 4.368 \\ 
 8000 5.268 \\ 
 9000 6.09 \\ 
10000 7.136 \\ 
11000 8.186 \\ 
12000 9.37 \\ 
13000 10.312 \\ 
14000 11.108 \\ 
15000 11.946 \\ 
16000 12.51 \\ 
17000 12.352 \\ 
18000 13.068 \\ 
19000 14.874 \\ 
20000 15.706 \\ 
};
\addplot[ mark=*, mark options={draw=colorsparsegsdo, fill=colorsparsegsdo}, mark size= 1pt, draw=colorsparsegsdo] table[row sep=crcr]{%
x	y\\
 1000 0.591872 \\ 
 2000 0.926 \\ 
 3000 1.512 \\ 
 4000 1.93 \\ 
 5000 2.872 \\ 
 6000 2.928 \\ 
 7000 3.276 \\ 
 8000 4.776 \\ 
 9000 5.276 \\ 
10000 4.838 \\ 
11000 6.482 \\ 
12000 6.196 \\ 
13000 6.628 \\ 
14000 7.602 \\ 
15000 8.89 \\ 
16000 8.392 \\ 
17000 8.476 \\ 
18000 8.202 \\ 
19000 9.224 \\ 
20000 10.65 \\ 
};

\addplot[ mark=*, mark options={draw=colornet, fill=colornet}, mark size= 1pt, draw=colornet] table[row sep=crcr]{%
x	y\\
 1000 0.452608 \\ 
 2000 0.854 \\ 
 3000 1.328 \\ 
 4000 1.664 \\ 
 5000 2.048 \\ 
 6000 2.418 \\ 
 7000 2.808 \\ 
 8000 3.202 \\ 
 9000 3.612 \\ 
10000 4.01 \\ 
11000 4.398 \\ 
12000 4.79 \\ 
13000 5.182 \\ 
14000 5.576 \\ 
15000 5.972 \\ 
16000 6.426 \\ 
17000 5.894 \\ 
18000 6.234 \\ 
19000 6.576 \\ 
20000 9.41 \\ 
};

\addlegendentry{full (\cite{marin2020correspondence},ours)}
\addlegendentry{$k=100$ (ours)}
\addlegendentry{$k=1$ (ours)}
\addlegendentry{only netw.}
\end{axis}

\end{tikzpicture}%
    \caption{Memory usage during non-stochastic training (red), stochastic training for different $k$ (green), and just calculating the network output (blue)}
    \label{fig:sparse:results}
    \end{subfigure}
    \hspace{0.2cm}
     \begin{subfigure}[t]{0.45\textwidth}
     \centering
    \definecolor{color1}{rgb}{0.00000, 0.00000, 0.60000}%
\definecolor{color2}{rgb}{0.00000, 0.60000, 0.00000}%
\definecolor{color3}{rgb}{0.60000, 0.00000, 0.00000}%
\definecolor{color4}{rgb}{0.00000, 0.60000, 0.60000}%
\definecolor{color5}{rgb}{0.60000, 0.00000, 0.60000}%

\begin{tikzpicture}

\begin{axis}[%
width=.5\textwidth,
height=.3\textwidth,
at={(0.758in,0.481in)},
scale only axis,
grid=both,
xmode=log,
grid style={line width=.1pt, draw=gray!50},
xlabel style={font=\color{white!15!black}},
x tick label style = {font=\footnotesize},
y tick label style = {font=\footnotesize},
x label style={at={(axis description cs:0.5,0.04)}, anchor=north},
y tick label style = {font=\footnotesize},
ytick = {-0.1,0,0.1},
ymax=0.11,
xtick = {0.5,1,2},
xticklabels =  {},
ylabel style={font=\color{white!15!black}},
ylabel={\small Rel. $\errorOne$},
xmin = 0.4,
xmax = 2.5,
y label style={at={(axis description cs:0.15,.5)},rotate=0,anchor=south},
axis background/.style={fill=white},
legend style={nodes={scale=0.8, transform shape}},
        legend pos=outer north east,
        legend cell align=left,
        smooth
]

\addplot[only marks, mark=*, mark options={draw=black!30!color1, fill=color1}, mark size= 2pt, draw=black!30!color1] table[row sep=crcr]{
 x	y \\
0.948936170212766   0.04589795227129096  \\ %
1.0   -0.04412209504857134  \\ %
1.0   -0.03374405112046995  \\ %
1.1106382978723404   -0.012200385327818072  \\ %
nan   0.05335739481056421  \\ %
};
\addplot[only marks, mark=*, mark options={draw=black!30!color2, fill=color2}, mark size= 2pt, draw=black!30!color2] table[row sep=crcr]{
 x	y \\
0.7500000000000001   0.07229313074549916  \\ %
0.7914364640883977   -0.09497900487487393  \\ %
1.0   -0.030274751406454217  \\ %
1.875   -0.07370986364119912  \\ %
};
\addplot[only marks, mark=*, mark options={draw=black!30!color3, fill=color3}, mark size= 2pt, draw=black!30!color3] table[row sep=crcr]{
 x	y \\
0.5655829596412555   0.07646404903302645  \\ %
0.6866591928251121   -0.022247205864867597  \\ %
1.0   -0.05346667923453825  \\ %
2.367152466367713   0.03357711875543661  \\ %
};
\addplot[only marks, mark=*, mark options={draw=black!30!color4, fill=color4}, mark size= 2pt, draw=black!30!color4] table[row sep=crcr]{
 x	y \\
0.5317261007868743   0.03909939204876126  \\ %
0.6325129750544115   -0.03151122660342041  \\ %
1.0   -0.0627688533720485  \\ %
};
\addplot[only marks, mark=*, mark options={draw=black!30!color5, fill=color5}, mark size= 2pt, draw=black!30!color5] table[row sep=crcr]{
 x	y \\
0.44066717596129357   0.07032005832326876  \\ %
0.5662083015024192   0.05830633326866892  \\ %
1.0   -0.003983793703339676  \\ %
};

\addlegendentry{$n = 1000$}
\addlegendentry{$n = 5000$}
\addlegendentry{$n = 10000$}
\addlegendentry{$n = 15000$}
\addlegendentry{$n = 20000$}
legend style={at={(0,-100.5)},anchor=west}
\end{axis}

\begin{axis}[%
width=.5\textwidth,
height=.3\textwidth,
at={(0.758in,0.481in)},
scale only axis,
yshift=-1.95cm,
xmax = 2.5,
xmin = 0.4,
xmode=log,
grid=both,
grid style={line width=.1pt, draw=gray!50},
xlabel style={font=\color{white!15!black}},
x tick label style = {font=\footnotesize},
y tick label style = {font=\footnotesize},
x label style={at={(axis description cs:0.5,0.04)}, anchor=north},
xlabel={\small Relative Mem. Cons.},
xtick = {0.5, 1 ,2},
xticklabels = {0.5, 1 ,2},
ytick = {-0.1,0,0.1},
ylabel style={font=\color{white!15!black}},
ylabel={\small Rel. $\errorTwo$},
y label style={at={(axis description cs:0.15,.5)},rotate=0,anchor=south},
axis background/.style={fill=white},
legend style={nodes={scale=0.8, transform shape}},
        legend pos=outer north east,
        legend cell align=left,
        smooth
]

\addplot[only marks, mark=*, mark options={draw=black!30!color1, fill=color1}, mark size= 2pt, draw=black!30!color1] table[row sep=crcr]{
 x	y \\
0.948936170212766   0.038428436824327415  \\ %
1.0   -0.11782416145409856  \\ %
1.0   -0.10438794644250296  \\ %
1.1106382978723404   -0.1349039952175519  \\ %
nan   -0.11435823640704557  \\ %
};
\addplot[only marks, mark=*, mark options={draw=black!30!color2, fill=color2}, mark size= 2pt, draw=black!30!color2] table[row sep=crcr]{
 x	y \\
0.7500000000000001   0.06899063031071308  \\ %
0.7914364640883977   -0.0770665136920835  \\ %
1.0   -0.12482272249539844  \\ %
1.875   -0.09104103015781995  \\ %
};
\addplot[only marks, mark=*, mark options={draw=black!30!color3, fill=color3}, mark size= 2pt, draw=black!30!color3] table[row sep=crcr]{
 x	y \\
0.5655829596412555   0.05191253609423983  \\ %
0.6866591928251121   -0.09450536175265888  \\ %
1.0   -0.0421416918162585  \\ %
2.367152466367713   -0.09135207797584696  \\ %
};
\addplot[only marks, mark=*, mark options={draw=black!30!color4, fill=color4}, mark size= 2pt, draw=black!30!color4] table[row sep=crcr]{
 x	y \\
0.5317261007868743   0.03418430063334617  \\ %
0.6325129750544115   -0.009322194896235324  \\ %
1.0   -0.10044899202493936  \\ %
};
\addplot[only marks, mark=*, mark options={draw=black!30!color5, fill=color5}, mark size= 2pt, draw=black!30!color5] table[row sep=crcr]{
 x	y \\
0.44066717596129357   0.10031338210660469  \\ %
0.5662083015024192   -0.048752479126254133  \\ %
1.0   -0.10564567968014699  \\ %
};

legend style={at={(0,-100.5)},anchor=west}
\end{axis}

\end{tikzpicture}%
    \caption{Relative memory - error trade-off for a varying number of vertices $(n)$}
    \label{fig:memory_tradeoff}
    \end{subfigure}
    \caption{(a) shows the memory consumption during training for shape matching for objects with a varying number of vertices $(n)$ and varying sparseness $(k)$. (b) shows the relative memory-error trade-off for varying sparseness (which causes the memory reduction), whereby the memory consumption is relative to $k=100$ and the errors are relative to full training by \cite{marin2020correspondence} for $n=1000$.}
\end{figure}

\section{Conclusion}
In this work, we proposed a strategy to represent permutation matrices by a low-rank matrix factorization followed by a nonlinearity and showed that by using the Kissing number theory, we can determine the minimum rank necessary for representing a permutation matrix of a given size, allowing for a more memory-efficient representation.  We validated this method with experiments on LAPs and QAPs as well as a real-world shape matching application and showed improvements in the latter. Additionally, we explored the potential of optimizing permutations stochastically to decrease memory usage, which opens the possibility of handling high-resolution data.

\paragraph{Limitations and Broader Impact}

Our method offers a promising solution to contribute positively to the environment by reducing the computational cost of a variety of problems involving permutation matrices. We do not see any ethical concerns associated with our approach itself. However, it is important to acknowledge a limitation of our method. For certain problem formulations, such as the Koopmans and Beckmann form QAPs, stochastic learning may not be feasible because the double occurrences of the permutation matrix make the stochastic computation not applicable. Moreover, our method requires devising a non-trivial, problem-specific adaptation.

\bibliographystyle{plain}
\bibliography{egbib}

\appendix

\definecolor{colornet}{rgb}{0.00000, 0.00000, 0.60000}%
\definecolor{colorsparsegsd}{rgb}{0.00000, 0.60000, 0.00000}%
\definecolor{colorfull}{rgb}{0.60000, 0.00000, 0.00000}%
\definecolor{color4}{rgb}{0.00000, 0.60000, 0.60000}%
\definecolor{color5}{rgb}{0.60000, 0.00000, 0.60000}%

\section{Handling Large Permutation Matrices}
Following our shape-matching experiments described in Sec. 4.5, we further visualize in Fig.~\ref{fig:sparse_est} how the proposed approach enables handling large problems that would have been infeasible otherwise.
The dashed red curve added to this figure (on top of the curves presented in Fig. 5a of the paper) corresponds to the estimated  memory the would have been required to accommodate full permutation matrices, as a function of problem size $n$. While our approach can accurately handle large problems with as much as $n=20,000$ vertices (green curves), running the equivalent experiments without it (red curves) would require prohibitively large amounts of memory ($\sim73.6~\text{GB}$, vs. $10.7~\text{GB}$ using $k=1$). For estimating memory values (dashed curve) we assume memory usage follows a $c\cdot n^2$ curve, and estimate the value for $c$ based on the full matrix experiments (solid red) we conducted for $n\leq11,000$.
\begin{figure}[!h]
    \centering
    \definecolor{colornet}{rgb}{0.00000, 0.00000, 0.60000}%
\definecolor{colorfull}{rgb}{0.60000, 0.00000, 0.00000}%
\definecolor{colorsparsegsd}{rgb}{0.00000, 0.60000, 0.00000}%
\definecolor{colorsparsegsdo}{rgb}{0.00000, 0.90000, 0.00000}%
\definecolor{color4}{rgb}{0.00000, 0.60000, 0.60000}%
\definecolor{color5}{rgb}{0.60000, 0.00000, 0.60000}%

\begin{tikzpicture}

\begin{axis}[%
width=.82\linewidth,
height=.42\linewidth,
at={(0.758in,0.481in)},
scale only axis,
xmin=0,
ymin=0,
grid=both,
grid style={line width=.1pt, draw=gray!50},
xlabel style={font=\color{white!15!black}},
x tick label style = {font=\footnotesize},
y tick label style = {font=\footnotesize},
x label style={at={(axis description cs:0.5,0.0)}, anchor=north},
xlabel={\small $n$},
ylabel={\small Memory consumprion (GB)},
legend style={nodes={scale=0.8, transform shape},at={(0.65,0.95)}},
        legend cell align=left,
        smooth,
        fill=white, fill opacity=0.6, text opacity=1
]

\addplot[ mark=*, mark options={draw=colorfull, fill=colorfull}, mark size= 1pt, draw=colorfull] table[row sep=crcr]{%
x	y\\
 1000 0.6696960000000001 \\ 
 2000 1.274 \\ 
 3000 1.952 \\ 
 4000 3.126 \\ 
 5000 4.546 \\ 
 6000 7.182 \\ 
 7000 9.5 \\ 
 8000 12.146 \\ 
 9000 15.098 \\ 
10000 18.378 \\ 
11000 21.962 \\ 
};

\addplot[dashed, mark=*, mark options={draw=colorfull, fill=colorfull}, mark size= 1pt, draw=colorfull] table[row sep=crcr]{%
x	y\\
1000 0.184 \\
2000 0.736 \\
3000 1.655 \\
4000 2.942 \\
5000 4.597 \\
6000 6.620 \\
7000 9.011 \\
8000 11.769 \\
9000 14.895 \\
10000 18.389 \\
11000 22.251 \\
12000 26.480 \\
13000 31.078 \\
14000 36.043 \\
15000 41.376 \\
16000 47.076 \\
17000 53.145 \\
18000 59.581 \\
19000 66.385 \\
20000 73.557 \\
};

\addplot[ mark=*, mark options={draw=colorsparsegsd, fill=colorsparsegsd}, mark size= 1pt, draw=colorsparsegsd] table[row sep=crcr]{%
x	y\\
 1000 0.48128 \\ 
 2000 0.9 \\ 
 3000 1.466 \\ 
 4000 2.262 \\ 
 5000 2.894 \\ 
 6000 3.71 \\ 
 7000 4.368 \\ 
 8000 5.268 \\ 
 9000 6.09 \\ 
10000 7.136 \\ 
11000 8.186 \\ 
12000 9.37 \\ 
13000 10.312 \\ 
14000 11.108 \\ 
15000 11.946 \\ 
16000 12.51 \\ 
17000 12.352 \\ 
18000 13.068 \\ 
19000 14.874 \\ 
20000 15.706 \\ 
};
\addplot[ mark=*, mark options={draw=colorsparsegsdo, fill=colorsparsegsdo}, mark size= 1pt, draw=colorsparsegsdo] table[row sep=crcr]{%
x	y\\
 1000 0.591872 \\ 
 2000 0.926 \\ 
 3000 1.512 \\ 
 4000 1.93 \\ 
 5000 2.872 \\ 
 6000 2.928 \\ 
 7000 3.276 \\ 
 8000 4.776 \\ 
 9000 5.276 \\ 
10000 4.838 \\ 
11000 6.482 \\ 
12000 6.196 \\ 
13000 6.628 \\ 
14000 7.602 \\ 
15000 8.89 \\ 
16000 8.392 \\ 
17000 8.476 \\ 
18000 8.202 \\ 
19000 9.224 \\ 
20000 10.65 \\ 
};

\addlegendentry{Full $P$ matrix}
\addlegendentry{Full $P$ matrix - estimate}
\addlegendentry{Our representation, stochastic training, $k=100$}
\addlegendentry{Our representation, stochastic training, $k=1$}
\end{axis}

\end{tikzpicture}
    \caption{
    \textbf{Memory savings.} Memory usage when training the shape matching network of [26] with different permutation matrix representations: Using full matrices (red) vs. using our stochastic training scheme with different sparseness levels (green). 
    See text for details.
    }
    \label{fig:sparse_est}
\end{figure}
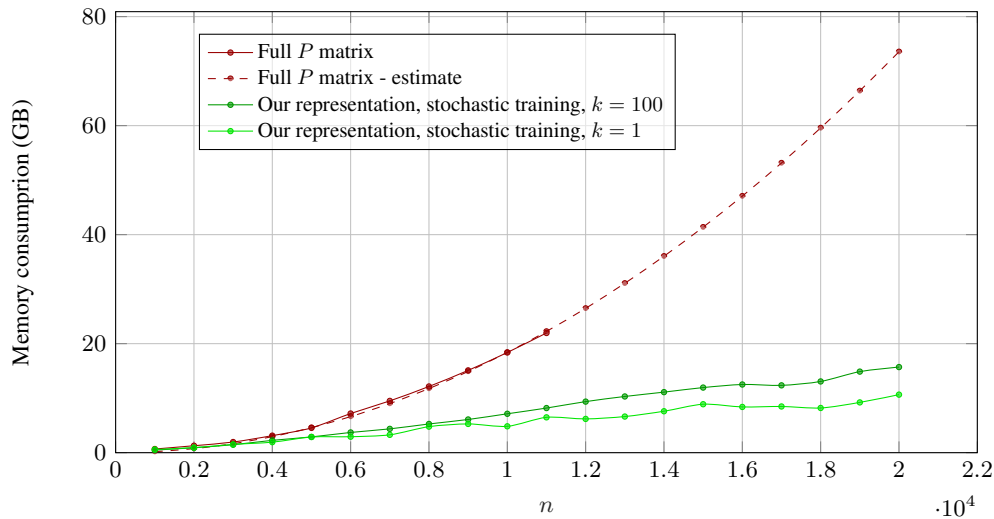

\end{document}